\documentclass{article}
\usepackage{spconf,amsmath,graphicx,hyperref}
\usepackage{url}            
\usepackage{booktabs}       
\usepackage{amsfonts}       
\usepackage{nicefrac}       
\usepackage{microtype}      
\usepackage{xcolor}         
\usepackage{amsthm}
\usepackage{comment}
\usepackage{algorithm, algorithmic}
\usepackage{mdframed}

\newtheorem{theorem}{Theorem}[section]
\newtheorem{proposition}{Proposition}[section]

\newtheorem{definition}{Definition}[section]
\newtheorem{assumption}{Assumption}[section]

\title{AL-CoLe: Augmented Lagrangian for Constrained Learning}
\name{Ignacio Boero, Ignacio Hounie, Alejandro Ribeiro}
\address{University of Pennsylvania}
\begin{document}
%
\maketitle
\begin{abstract}
Despite the non-convexity of most modern machine learning parameterizations, Lagrangian duality has become a popular tool for addressing constrained learning problems. We revisit Augmented Lagrangian methods, which aim to mitigate the duality gap in non-convex settings while requiring only minimal modifications, and have remained comparably unexplored in constrained learning settings. We establish strong duality results under mild conditions, prove convergence of dual ascent algorithms to feasible and optimal primal solutions, and provide PAC-style generalization guarantees. Finally, we demonstrate its effectiveness on fairness constrained classification tasks.
\end{abstract}
\section{Introduction}
\label{sec:intro}

The widespread adoption of machine learning in critical domains has underscored the need to curtail undesirable behavior. A natural approach is to extend empirical risk minimization (ERM) with additional statistical risks—such as fairness, robustness, or safety—as explicit constraints. Doing so, however, raises challenges for both optimization and generalization.  On the optimization side, mini-batch updates make feasibility during training intractable, limiting interior-point methods, while the non-convexity of neural networks prevents strong duality, weakening guarantees of feasibility and optimality for dual methods. On the generalization side, constraint estimation error impacts both feasibility and optimality, the latter depending on sensitivity to perturbations.  

To circumvent the lack of strong duality, prior work has treated non-convex constrained problems as parameterizations of convex functional ones, where duality holds. This yields bounds on the duality gap in terms of approximation error and allows extending PAC-learnability~\cite{valiant1984theory} to constrained settings~\cite{chamon2020probably, chamon2022constrained}. More recently,~\cite{elenter2024near} showed that primal solutions obtained from dual iterates can achieve near-feasible, near-optimal solutions in over-parameterized regimes.
Although a growing body of work has studied constrained learning via dual formulations~\cite{ diana2021minimax, bai2022achievingzeroconstraintviolation, ConstrainedFederated, ConstrainedActive, gallegoposada2022controlled, 10480186, manolache2025learningapproximatelyequivariantnetworks, navid}, these approaches typically rely on (non-augmented) dual ascent.

Augmented Lagrangian (AL) methods were originally proposed to close the duality gap in nonconvex optimization~\cite{RockaDual}. 
Although they have appeared sporadically in machine-learning applications—such as diffusion models~\cite{al-fioretto-diffusion-robots}, physics-informed neural networks~\cite{al-pinn}, pruning~\cite{al-prunning}, learning-to-optimize~\cite{al-fioretto-l2o}, SVMs~\cite{al-svm}, reinforcement learning~\cite{al-rl}, and stochastic optimization~\cite{al-stochastic-opt}—these uses are largely empirical and lack a theoretical foundation. 
In particular, they neither define constraints on statistical risk functionals nor study the generalization properties of empirical AL methods under such constraints.

In this work, we study AL methods for constrained learning and develop a generalization theory within the PAC-Constrained framework~\cite{chamon2020probably}. We prove that, under mild assumptions, strong duality holds between the primal and extended dual problems, eliminating approximation error. Crucially, we show that constrained learning problems are \emph{PACC-learnable}, rather than merely near-PACC-learnable as previously believed. Moreover, augmented dual algorithms—defined by sequences of unconstrained subproblems—yield primal iterates converging to feasible and optimal solutions. Finally, we validate our results empirically on a fairness-constrained classification task, comparing augmented and standard dual methods.

\begin{figure}[t!]
    \centering
    \includegraphics[width=0.95\linewidth]{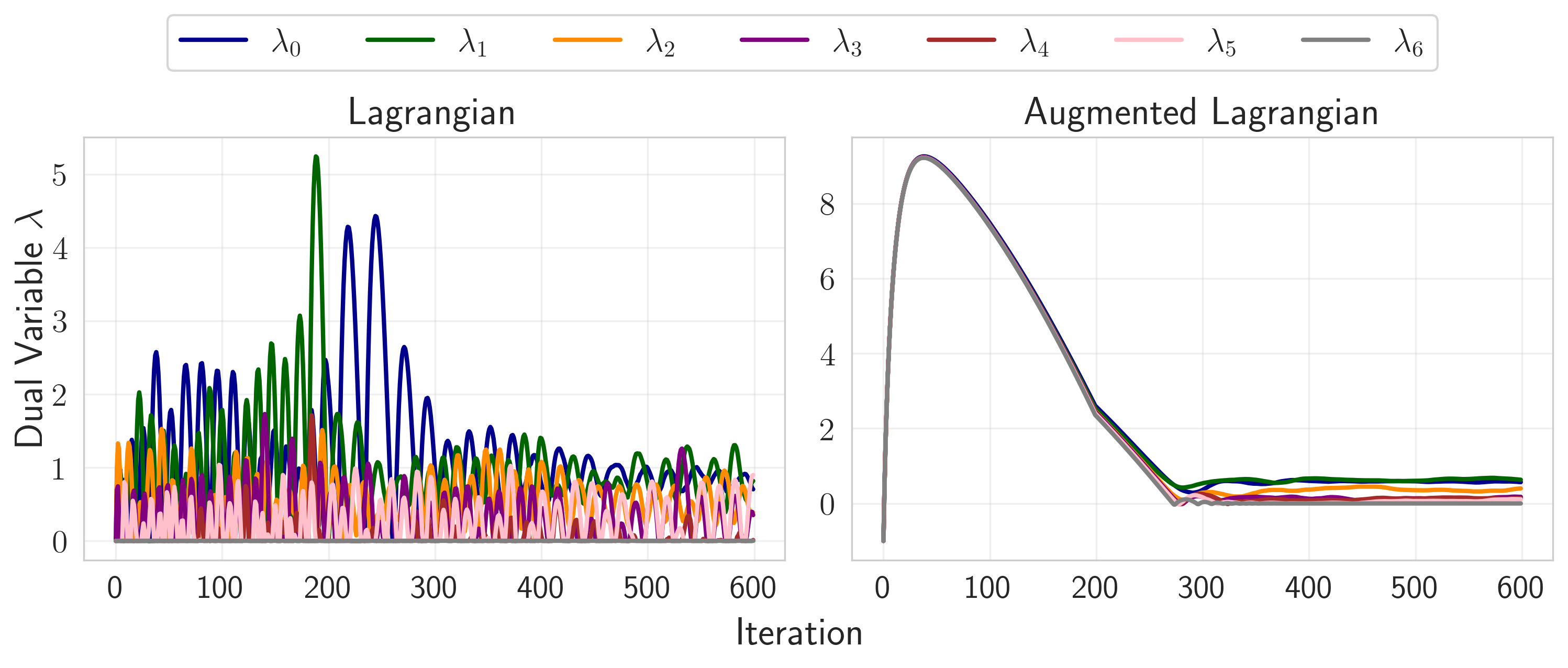}
    \caption{Dynamics of lagrange multipliers $\lambda$ training iterations, for both the augmented Lagrangian and regular Lagrangian methods. Both methods use the same hyperparameters, including the dual learning rate.}
    \label{fig:dynamics-lambda}
\end{figure}

\section{Augmented Lagrangian for constrained learning}

\subsection{Statistical Constrained Learning}

Constrained learning problems typically consist of the minimization of a risk functional under constraints over distinct risk functionals and data distributions, namely
\begin{equation}
\tag{P}
\label{eq:constrained_risk}
\begin{aligned}
 & \text{minimize} \; && {\ell}_0 (f_{\theta}) \quad \text{over all } \theta \in \Theta\\
& \; \text{s.t.} && \ell_i (f_{\theta})
 \le 0, \;  i \in \{1,\dots,m\} 
\end{aligned}
\end{equation}
with the risk functionals defined as
\begin{equation}
{\ell}_i (f_{\theta}) := \mathbb{E}_{(x,y)\sim \mathcal{D}_i}\!\left[\,\bar{\ell}_i\!\big(f_\theta(x),\, y\big)\,\right], \;  \forall i \in \{0,\dots,m\} 
\end{equation}
where $f_\theta(x): \mathcal{X} \to \mathcal{Y}$ is a function parametrized by the vector $\theta \subseteq \Theta \in \mathbb{R}^p$, $\bar{\ell}_i$ denotes the sample-wise loss and $\ell_i$ is its corresponding expected risk under distribution $\mathcal{D}_i$. Here $\ell_0$ represents the risk we wish to minimize (e.g. cross entropy) and $\ell_i$ the risks we wish to constrain (e.g. counterfactual fairness). For most common parametrization choices (e.g. neural networks),~\ref{eq:constrained_risk} is a non-convex constrained optimization problem. 

\subsection{Augmented Lagrangian and Duality}

Augmented Lagrangian methods for inequality-constrained problems were first introduced in \cite{RockaDual}. For problem \ref{eq:constrained_risk} the augmented Lagrangian is defined as
\begin{equation}
\label{eq:aug_lagrangian}
\begin{aligned}
L(\theta,\lambda,\alpha) = \ell_0(f_{\theta}) \;+\; \alpha \sum_{i=1}^m \Psi(l_i(f_\theta), \frac{\lambda_i}{\alpha})
\end{aligned}
\end{equation}
for $(\lambda,\alpha) \in T \subseteq \mathbb{R}^m \times (0,+\infty)$ where
\begin{equation}
    \Psi(x,y) = [\text{max}^2 \left\{0,2x + y \right\} - y^2]/4
\end{equation}

The corresponding augmented dual problem is defined as
\begin{equation}
\tag{D}
\label{eq:dual_problem}
\begin{aligned}
 \text{maximize} &\quad  g(\lambda, \alpha) \quad \text{over all} \;(\lambda, \alpha) \in T, \; \text{where} \\
 &g(\lambda, \alpha) := \inf_{\theta \in \Theta} L(\theta,\lambda, \alpha) 
\end{aligned}
\end{equation}

Analogous to the standard (non-augmented) dual problem, weak duality and concavity hold [\ref{sec:dual_prop}].
In contrast, the augmented Lagrangian enables strong duality even in non-convex settings. 
We now state the assumptions required to guarantee strong duality in our problem.

\begin{assumption}[Bounded Objective]
\label{ass:l0_bounded}
The function $\ell_0$ is bounded below.
\end{assumption}

\begin{assumption}[Closed Domain]
\label{ass:theta_closed}
The set $\Theta$ is closed (e.g., $\Theta=\mathbb{R}^p$).
\end{assumption}

\begin{assumption}[Continuity]
\label{ass:l_continuity}
The mapping $\theta \mapsto f_{\theta}$ is continuous, and the sample-wise losses 
$\bar{\ell}_i(\cdot,y)$ are continuous and uniformly bounded in~$y$.
\end{assumption}

\begin{assumption}[Compact Feasible Level Set]
\label{ass:feasible_compact}
There exist $u>0$ and $\bar{p} > \inf(P)$ such that
\[
\{\theta \in \Theta : \ell_0(\theta) \le \bar{p}, \ \ell_i(\theta) \le u,\ \forall i=1,\dots,m\}
\]
is compact.
\end{assumption}

Assumption~\ref{ass:l0_bounded} holds whenever the sample-wise losses $\bar{\ell}_0$ are bounded below, 
as is the case for standard losses such as MSE or cross-entropy. 
Assumptions~\ref{ass:theta_closed} and~\ref{ass:l_continuity} impose only mild regularity on the parameterization mapping and the constraint functions. 
Finally, Assumption~\ref{ass:feasible_compact} generalizes Slater’s condition to non-convex settings; 
in the convex case, Slater’s condition is sufficient for it to hold.

\begin{theorem}[Strong Duality]
\label{thm:strong_duality_main}
If Assumptions~\ref{ass:l0_bounded}--\ref{ass:feasible_compact} hold, then strong duality holds for problem~\ref{eq:constrained_risk}, i.e.,
\[
\inf(P) = \sup(D).
\]
\end{theorem}

\begin{proof}
See Appendix \ref{sec:strong_dual}.
\end{proof}

Although concavity and strong duality imply that solving $P$ should be straightforward---since subgradient methods address the dual and, in turn, the primal---one difficulty remains: the dual optimum may exist only as a limit. Thus, establishing the existence of bounded dual variables $(\alpha^*, \lambda^*)$ such that $g(\alpha^*, \lambda^*) = \sup(D) = \inf(P)$ is crucial. In non-augmented dual formulations, the boundedness of the multipliers $\lambda^*$ hinges on first-order sufficiency (e.g., the KKT conditions) being necessary. This necessity, in turn, is guaranteed by assuming the objective and constraint functions are $C^1$ and that a first-order constraint qualification holds---for instance, the Linear Independence Constraint Qualification (LICQ).

\begin{assumption}[Linear Independence Constraint Qualification]
\label{ass:licq}
Problem $P$ satisfies the Linear Independence Constraint Qualification (LICQ). That is, the set of active constraint gradients 
$\{\nabla_{\theta} \ell_i(f_\theta^*)\}$ 
is linearly independent at the optimal point $\theta^*$.
\end{assumption}

For the augmented dual problem, it is also required that the second-order sufficient conditions be necessary, together with the assumption that both the constraint and objective functions are $\mathcal{C}^2$.

\begin{assumption}[Smoothness]
\label{ass:c2}
Each function $\ell_i(f_{\theta})$ is twice continuously differentiable, i.e., $\mathcal{C}^2$.
\end{assumption}

\begin{assumption}[Second-Order Sufficient Condition]
\label{ass:sosc}
Let $\bar{\theta} \in \Theta$ and $\bar{\lambda} \in \mathbb{R}^+$ be such that the KKT conditions hold, i.e.,
$\nabla_\theta L(\bar{\theta},\bar{\lambda},0) = 0$ and $\nabla_\lambda L(\bar{\theta},\bar{\lambda},0) = 0$. 
If Assumption~\ref{ass:licq} holds, such a pair $(\bar{\theta},\bar{\lambda})$ exists. Define
$H := \nabla_\theta^2 L(\bar{\theta},\bar{\lambda},0)$, 
$I_0 := \{\,i: f_i(\bar{\theta})=0,\ \lambda_i>0\,\}$,
$I_1 := \{\,i: f_i(\bar{\theta})=0,\ \lambda_i=0\,\}.$
Then the second-order sufficient condition requires
\begin{equation}
z^\top H z \;>\; 0 \quad \forall z \in C(\bar{\theta}), \quad \text{where}
\end{equation}
\[
C(\bar{\theta}) = \{\, z : 
\begin{cases}
\nabla f_i(\bar{\theta})^\top z = 0, & i \in I_0, \\ 
\nabla f_i(\bar{\theta})^\top z \le 0, & i \in I_1
\end{cases}\}.
\]
\end{assumption}

Assumption~\ref{ass:c2} enforces smoothness of both the loss functions and the parametrization mapping. 
It is satisfied, for instance, when employing neural networks with $\mathcal{C}^2$ activations (e.g., sigmoid) and smooth loss functions (e.g., MSE). 
In contrast, Assumptions~\ref{ass:sosc} and~\ref{ass:licq} are less transparent, as their validity depends on the first- and second-order properties of the problem at the optimum. 
Nevertheless, these conditions are typically assumed to hold. 
For a detailed discussion of sufficient conditions and constraint qualifications, see~\cite{con_qual}.

\begin{theorem}[Optimal Dual Solutions]
\label{thm:optimal_dual_main}
Suppose Assumptions~\ref{ass:l0_bounded}--\ref{ass:sosc} hold. 
Then there exists a pair $(\bar{\lambda}, \bar{\alpha})$ such that 
\[
g(\bar{\lambda},\bar{\alpha}) = \sup(D) = \inf(P).
\]
\end{theorem}

\begin{proof}
See Appendix \ref{sec:optimal_dual}.
\end{proof}

Theorem~\ref{thm:optimal_dual} establishes the existence of bounded dual variables, provided that both the first- and second-order conditions are satisfied.

\subsection{Empirical Dual Learning}

The dual problem~\ref{eq:dual_problem} involves statistical functionals over unknown distributions. 
As customary in machine learning, we replace expected risks $\ell_i(f_\theta)$ with sample means 
$\hat\ell_i(f_\theta) := \tfrac{1}{N}\sum_{n=1}^N \bar{\ell}_i(f_{\theta}(x_n), y_n)$. 
The resulting constrained \emph{empirical} risk minimization problem is
\begin{equation}
\tag{$\hat{P}$}
\label{eq:emp_constrained_risk}
\begin{aligned}
& \min_{\theta \in \Theta} \; \hat{\ell}_0(f_{\theta}) 
&& \text{s.t. } \hat{\ell}_i(f_{\theta}) \le 0, \; i=1,\dots,m,
\end{aligned}
\end{equation}
with dual
\begin{equation}
\tag{$\hat{D}$}
\label{eq:emp_dual_problem}
\max_{(\lambda,\alpha)\in T} \; \hat{g}(\lambda,\alpha)
:= \inf_{\theta \in \Theta} \Bigl[\hat{\ell}_0(f_{\theta}) 
+ \alpha \sum_{i=1}^m \Psi\!\left(\hat{\ell}_i(f_\theta), \tfrac{\lambda_i}{\alpha}\right)\Bigr].
\end{equation}
We require strong duality to hold for the empirical problem as well. 
To avoid redundancy, Assumptions~\ref{ass:licq}, \ref{ass:sosc} and Theorem~\ref{thm:optimal_dual} are not repeated; 
with minor modifications, they ensure $\inf(\hat{P})=\sup(\hat{D})$ and the existence of 
$(\hat{\lambda},\hat{\alpha})$ such that $\hat{g}(\hat{\lambda},\hat{\alpha})=\sup(\hat{D})$.

The standard PAC-learning framework bounds the estimation gap between the statistical risk of the empirical solution and the optimum, i.e., $|P^* - \ell_0(f_{\hat{\theta}^\star})|$. 
Our constrained setting requires extending this framework to also bound infeasibility, for which we adopt the (agnostic) PAC-Constrained framework of~\cite{chamon2020probably}. 
As in any agnostic PAC framework, uniform convergence is required. 

\begin{assumption}[Uniform Convergence]\label{ass:unif-conv}
For $i=0,\dots,m$, there exists $\zeta_i(N,\delta)\!\geq\!0$, monotonically decreasing in $N$, such that
\[
|\,\ell_i(f_{\theta}) - \hat{\ell}_i(f_{\theta})| \;\le\; \zeta_i(N,\delta),
\]
for all $\theta \in \Theta$, with probability $1-\delta$ over i.i.d.~samples $(x_n,y_n)\!\sim\!\mathcal{D}_i$.
\end{assumption}
Our main theoretical result is given in Theorem~\ref{thm:generalisation}, which establishes bounds on both suboptimality and infeasibility of~\ref{eq:emp_dual_problem} as an approximation of~\ref{eq:constrained_risk}.

\begin{theorem}[PAC-Constrained Dual Learning]
\label{thm:generalisation_main}
Let $(\hat{\lambda}^\star,\hat{\alpha}^\star)$ be a solution of the empirical augmented dual problem~\ref{eq:emp_dual_problem}. 
Under Assumptions~\ref{ass:l0_bounded}--\ref{ass:unif-conv}, the solution of~\ref{eq:emp_constrained_risk} satisfies 
$\hat{\theta}^\star \in \arg\min_{\theta \in \Theta} \hat{L}(\theta,\hat{\lambda}^\star,\hat{\alpha}^\star)$ and, with probability $1-(2m+1)\delta$,
\begin{align}
|P^\star - \ell_0(f_{\hat{\theta}^\star})| 
&\;\le\; (2+\Delta_{\lambda})\bar{\zeta} + \Delta_{\alpha}m\bar{\zeta}^2, 
\label{eq:generalisation-obj_main} \\
\ell_i(f_{\hat{\theta}^\star}) 
&\;\le\; \zeta_i(N_i,\delta), \quad \forall i=1,\dots,m,
\label{eq:generalisation-constr_main}
\end{align}
where $P^\star$ is the optimal value of~\ref{eq:constrained_risk}, 
$\bar{\zeta} = \max_i \zeta_i(N_i,\delta)$, 
$\Delta_\lambda = \max(\|\lambda^*\|_1,\|\hat{\lambda}^*\|_1)$, 
and $\Delta_\alpha = \max(\alpha^*,\hat{\alpha}^*)$.
\end{theorem}

\begin{proof}
See Appendix \ref{sec:gener}.
\end{proof}

Inequalities~\ref{eq:generalisation-obj}--\ref{eq:generalisation-constr} 
correspond exactly to the conditions for PACC-learnability~\cite[Def.~2]{chamon2020probably}. 
Hence, Theorem~\ref{thm:generalisation} establishes that problem~\ref{eq:constrained_risk} is \emph{PACC-learnable}. 
This is a stronger guarantee than prior results~\cite{chamon2020probably,chamon2022constrained,elenter2024near}, 
which only achieved near-PACC-learnability~\cite[Def.~3]{chamon2020probably}, 
since their bound for $|P^\star - \ell_0(f_{\hat{\theta}^\star})|$ included a constant term~$\epsilon_0$ independent of~$\bar{\zeta}$. 
In contrast, our augmented Lagrangian approach ensures that, with sufficient samples, the gap can be driven arbitrarily close to zero.

\section{Augmented dual ascent}
In the previous section we concluded that~\ref{eq:emp_dual_problem} provides an attractive approach to obtain a PACC solution of~\ref{eq:constrained_risk}, since it reduces the task to solving a concave, unconstrained problem. 
However, even if $(\lambda_k,\alpha_k)\!\to\!(\lambda^*,\alpha^*)$, the associated primal sequence 
$\{\theta_k^\star \in \arg\min_{\theta} \hat{L}(\theta,\lambda_k,\alpha_k)\}$ 
need not be asymptotically minimizing for~\ref{eq:emp_constrained_risk}. 
Formally, an \emph{asymptotically minimizing sequence} $\{\theta_k\}_{k=0}^\infty$ must satisfy: 
\begin{align}
\limsup_{k\to\infty} \hat{\ell}_i(f_{\theta_k}) &\;\le\; 0, \quad \forall i=1,\dots,m, 
\label{eq:asymp-feas} \\
\limsup_{k\to\infty} \hat{\ell}_0(f_{\theta_k}) &\;=\; \inf(\hat{P}). 
\label{eq:asymp-opt}
\end{align}
This drawback is typical of non-augmented dual methods, even in convex settings, where randomization or averaging over primal iterates is required. 
In contrast, the augmented Lagrangian avoids this phenomenon entirely. 
The next theorem formalizes this property.

\begin{theorem}[Primal Convergence]
\label{thm:primal_convergence}
Let $\delta >0$ and $(\lambda_k,\alpha_k)$ satisfy 
$\lim_{k\to\infty} \hat{g}(\lambda_k,\alpha_k - \delta) = \sup(\hat{D}) < \infty$. Let $\theta_k$ obey
\[
\hat{L}(\theta_k,\lambda_k,\alpha_k) 
\le \inf_{\theta\in \Theta} \hat{L}(\theta,\lambda_k,\alpha_k) + \epsilon_k,
\quad \epsilon_k \to 0.
\]
Then $(\theta_k)$ is asymptotically feasible. 
If $(\lambda_k)$ is bounded, $(\theta_k)$ is also asymptotically minimizing for~($\hat{P}$).
\end{theorem}
\begin{proof}
    See ~\cite[Theorem 3]{RockaDual}
\end{proof}

Theorem~\ref{thm:primal_convergence} ensures that any convergent sequence of dual iterates induces a convergent sequence of primal iterates. This is crucial in practice: any stopping point of the dual method yields a primal iterate close to optimal, without the need for averaging or randomization over iterates. To solve~\ref{eq:emp_dual_problem}, we use the \emph{increased-shifted penalty method}~\cite[Sec 7]{cone_proj}. The name refers to the periodic increase of~$\alpha$ and corrective shifts of~$\lambda$. 

\begin{align}
\lambda_{k+1} &=
\begin{cases}
\lambda_k + \tfrac{1}{\alpha_k}\,\ell(\theta_k), & \ell(\theta_k) \ge -\tfrac{\lambda_k}{2\alpha_k}, \\[4pt]
\lambda_k - \tfrac{2\lambda_k}{\alpha_k}, & \ell(\theta_k) \le -\tfrac{\lambda_k}{2\alpha_k},
\end{cases} \label{eq:update-lambda} \\[6pt]
\alpha_{k+1} &=
\begin{cases}
\alpha_k \cdot c, & \text{every $I$ iterations}, \\[4pt]
\alpha_k, & \text{otherwise},
\end{cases} \label{eq:update-alpha} \\[6pt]
\theta_{k+1} &= \arg\min_{\theta}\, L(\theta,\lambda_{k+1},\alpha_{k+1}). 
\label{eq:update-theta}
\end{align}

\section{Numerical results}

We consider learning classifiers that remain insensitive to protected attributes despite their correlation with predictive features using the COMPAS recidivism dataset~\cite{COMPAS}, where the goal is to predict recidivism while controlling for gender and racial bias.
In this binary classification task, the objective loss is cross-entropy $\bar{\ell}_0(\hat{y}, y) = -\log [\hat{y}]_y$. We then partition the input as $x = [\tilde{x}, z]$, where $z$ contains the protected features and $\tilde{x}$ contains the remaining input features. Following prior work~\cite{elenter2024near, chamon2020probably} we use $D_\text{KL}\left(f_{\theta}(\tilde{x}, z) \,\|\, f_{\theta}(\tilde{x},\rho_i (z))\right)\leq c, \quad i=1,\ldots,m$ as constraints,
where $\rho_i$ are transformations representing single-variable modifications of protected attributes, $c > 0$ is the desired sensitivity threshold, and $m$ is the number of constraints. Each constraint enforces approximate invariance to changes in protected features. This formulation aligns with the notion of average counterfactual fairness from~\cite[Definition 5]{kusner2018counterfactual}.

We evaluate three algorithmic approaches approaches: (i) an \emph{unconstrained} predictor trained solely with cross-entropy loss $\tilde{\ell}_0(\hat{y}, y)$, (ii) a fairness-constrained predictor trained using the Lagrangian dual method (iii) the fairness-constrained predictor using the Augmented Lagrangian dual method. Both the Augmented and regular Lagrangian utilize the same hyperparameters. For the regular Lagrangian we present results for both the final iterate $f_{\theta}(T)$ and a \emph{randomized} predictor that samples uniformly from primal iterates $\{f_{\theta}(t)\}_{t=t_0}^T$ for each prediction.

\textbf{Results.} As shown in figure~\ref{fig:dynamics-lambda}, the augmented Lagrangian leads to smoother dual variable dynamics,  effectively mitigating the oscilations and cyclo-stationary behaviour that arises in standard Lagrangian dual learning. As shown in figure~\ref{fig:dynamics-slack}, which depicts the dynamics of constraint satisfaction across optimization steps, the effect is more pronounced in more challenging constraints (right plot) which show larger oscillations between iterates. Furthermore, the last iterate of the augmented Lagrangian method attains both slightly better performance in terms of test accuracy while improving constraint satisfaction, compared to to the last iterate of the regular Lagrangian. As shown in Figure~\ref{fig:accuracy}, it improves invariance with respect to gender; the rate with which predictions change when this attribute is flipped changes from $7.6\%$ to $1.5\%$. Although randomizing the regular Lagrangian solution over iterates yields accuracy and fairness metrics comparable to the augmented Lagrangian last iterate, randomization is impractical due to the need to store and deploy multiple models at inference time.

\begin{figure}[t!]
    \centering
    \includegraphics[width=0.95\linewidth]{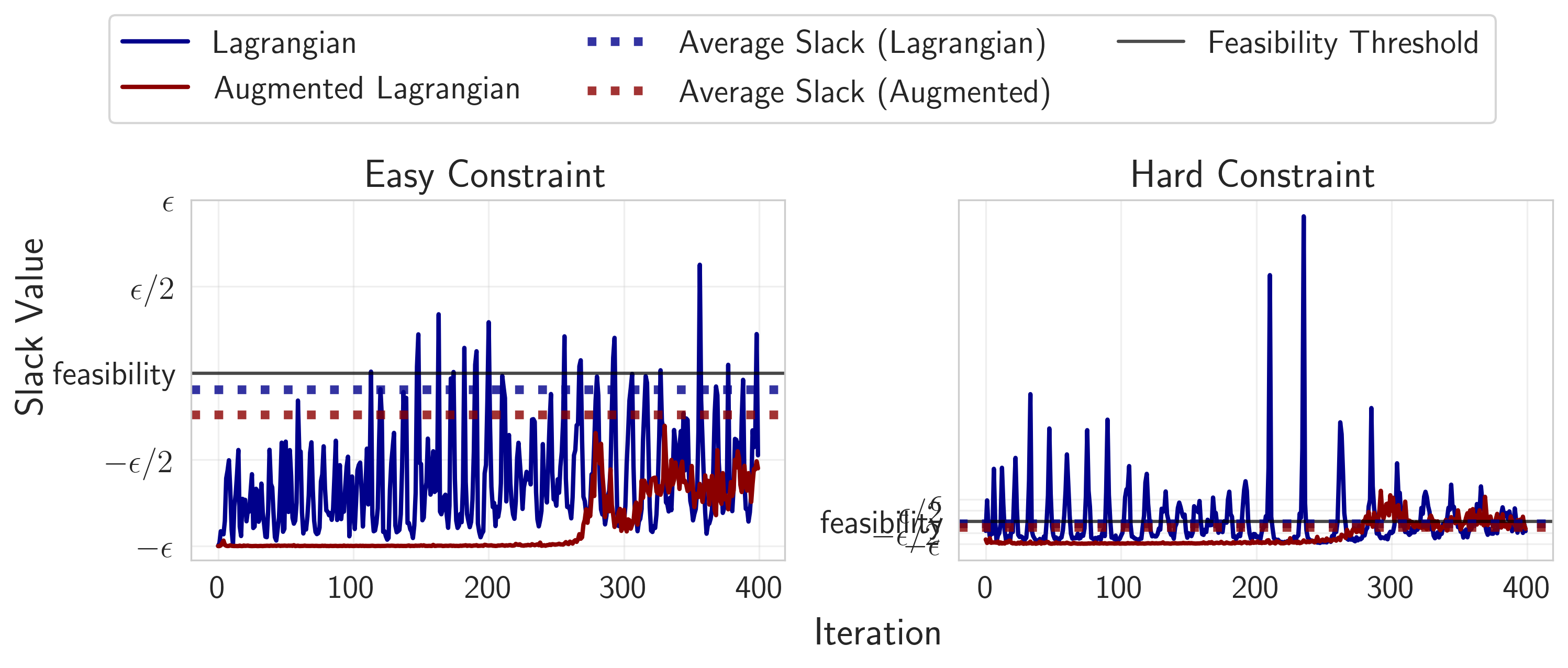}
\caption{\textbf{(Left)}: Evolution of constraint slacks $s_i^t$ during training for easier constraints. 
\textbf{(Right)}: Same for harder constraints. 
A constraint is satisfied when $s_i^t<0$. 
We compare augmented and standard Lagrangian methods.}    \label{fig:dynamics-slack}
\end{figure}

\begin{figure}[t!]
    \centering
    \includegraphics[width=0.95\linewidth]{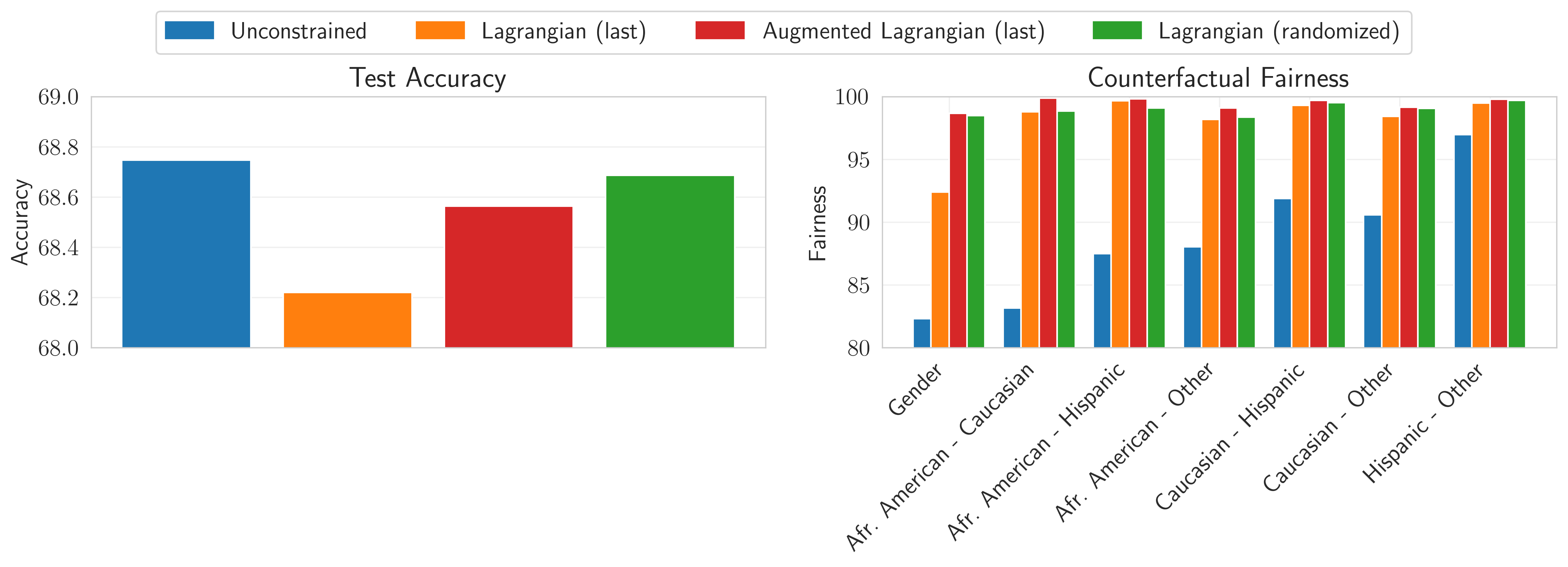}
    \caption{\textbf{(Left)}: Classification accuracy over held-out data for unconstrained and constrained methods. \textbf{(Right)}: Counterfactual fairness measured as the fraction of predictions that do not change after modifying a sensitive attribute.}
    \label{fig:accuracy}
\end{figure}

\section{Conclusions}

This work expands the literature on algorithms for constrained learning. 
Given the non-convex nature of these problems, we proposed the use of the augmented Lagrangian, 
which closes the duality gap and yields stronger guarantees for the estimation error. 
In particular, we established that constrained learning problems are \emph{PACC-learnable}, 
rather than merely near-PACC-learnable as previously believed. 
Moreover, the augmented Lagrangian ensures that the sequence of primal iterates is asymptotically minimizing, 
a property of practical relevance confirmed by our experiments: 
dual variables remained significantly more stable during training, 
and the last primal iterate of the augmented dual consistently outperformed that of the non-augmented dual.


\newpage
\appendix
\section{Duality of the Augmented Lagrangian}


\subsection{Properties of the dual problem}
\label{sec:dual_prop}

\begin{proposition}[Weak Duality]
\label{thm:weak_d}
For the (augmented) dual problem \ref{eq:dual_problem}, weak duality holds. Namely,
\begin{equation}
\label{eq:perturbation}
\begin{aligned}
\inf (P) \geq \sup (D)
\end{aligned}
\end{equation}
\end{proposition}

\begin{proof}
Because $\Psi(x,y)\le 0$ for all $x\le 0$, for any feasible $\theta\in\Theta$ and any $(\lambda,\alpha)\in\mathcal{T}$ we have
\[
\ell_0(f_{\theta}) \geq \ell_0(f_{\theta}) \;+\; \alpha \sum_{i=1}^m \Psi(l_i(f_\theta), \tfrac{\lambda_i}{\alpha}).
\]
In particular,
\[
\begin{aligned}
\inf (P) = \inf_{\theta \in \mathcal{F}}  \ell_0(f_{\theta}) &\geq \inf_{\theta \in \mathcal{F} } \Big(\ell_0(f_{\theta}) \;+\; \alpha \sum_{i=1}^m \Psi(l_i(f_\theta), \tfrac{\lambda_i}{\alpha}) \Big) \\
& \geq \inf_{\theta \in \Theta} \Big(\ell_0(f_{\theta}) \;+\; \alpha \sum_{i=1}^m \Psi(l_i(f_\theta), \tfrac{\lambda_i}{\alpha}) \Big)\\
&= g(\lambda, \alpha).
\end{aligned}
\]
Here $\mathcal{F}=\{\theta \in \Theta: \ell_i(f_\theta) \le 0,\; i=1,\dots,m\}$ denotes the feasible set, and the second inequality uses $\mathcal{F}\subseteq\Theta$. Since the bound holds for every $(\lambda,\alpha)\in\mathcal{T}$, it holds in particular at the supremum, i.e., $\sup_{(\lambda,\alpha)\in\mathcal{T}} g(\lambda,\alpha)=\sup(D)$.
\end{proof}

\begin{proposition}[Concavity of the dual]
\label{thm:conc_dual}
The functions $g(\lambda,r)$ and $L(x,\lambda,r)$ are concave and upper semi-continuous for $(\lambda,r)\in T$ and non-decreasing for $r$.
\end{proposition}
\begin{proof}
~\cite[Theorem 1]{RockaDual}
\end{proof}

\subsection{Strong duality}
\label{sec:strong_dual}
We first recall the perturbation functions. Let $p:\mathbb{R}^m\to\mathbb{R}$ denote the perturbation of \ref{eq:constrained_risk}:
\begin{equation}
\tag{$P_u$}
\label{eq:pert_func}
\begin{aligned}
    p(u) = &\min_{\theta \in \Theta} {\ell}_0 (f_{\theta}) \quad \\
 \; \text{s.t.} \quad  &\ell_i (f_{\theta})
 \le u_i, \;  i \in \{1,\dots,m\} 
\end{aligned}
\end{equation}
Likewise, $\hat p:\mathbb{R}^m\to\mathbb{R}$ is the perturbation of \ref{eq:emp_constrained_risk},
\begin{equation}
\tag{$\hat P_u$}
\label{eq:emp_pert_func}
\begin{aligned}
    \hat p(u) = &\min_{\theta \in \Theta} {\hat \ell}_0 (f_{\theta}) \quad \\
 \; \text{s.t.} \quad  &\hat \ell_i (f_{\theta})
 \le u_i, \;  i \in \{1,\dots,m\} 
\end{aligned}
\end{equation}

The next two properties are the ones we show to be equivalent to strong duality.

\begin{definition}[Growth Condition]
\label{def:growth}
Problem \ref{eq:constrained_risk} satisfies the \emph{growth condition} if there exists $\alpha>0$ such that
\begin{equation}
\label{eq:growth}
    \begin{aligned}
        L(\theta,0,\alpha) = \ell_0(f_\theta) + \alpha \sum_{i=1}^m \max [0, \ell_i(f_\theta)]^2
    \end{aligned}
\end{equation}
is bounded below as a function of $\theta$.
\end{definition}

\begin{definition}[Stability of Degree~0]
\label{def:stability}
Problem \ref{eq:constrained_risk} is \emph{stable of degree~0} if $p(u)$ is lower semicontinuous at $u=0$, i.e.,
\[
\lim_{u \downarrow 0} p(u) = p(0).
\]
\end{definition}

\begin{theorem}[Strong Duality]
\label{thm:strong_duality}
If Assumptions~\ref{ass:l0_bounded}--\ref{ass:feasible_compact} hold, then strong duality holds for problem~\ref{eq:constrained_risk}, i.e.,
\[
\inf(P) = \sup(D).
\]
\end{theorem}

\begin{proof}
By Assumption~\ref{ass:l0_bounded}, $\ell_0(f_\theta)$ is bounded below; let $\bar{\ell_0}$ be a lower bound. Since
\[
\alpha \sum_{i=1}^m \max [0, \ell_i(f_\theta)]^2 \geq 0
\]
for any $\alpha>0$, it follows that $L(\theta,\lambda,\alpha)$ is bounded below by $\bar{\ell_0}$, so the growth condition holds.

Next, by Assumption~\ref{ass:l_continuity}, the maps $\theta\mapsto f_\theta$ and $u\mapsto \bar{\ell}_i(u,y)$ are continuous, and $\bar{\ell}_i$ is uniformly bounded in $y$. Hence, by the bounded convergence theorem, $\theta\mapsto\ell_i(f_\theta)$ is continuous for each $i$. Together with Assumptions~\ref{ass:theta_closed} and \ref{ass:feasible_compact}, Berge’s Maximum Theorem applies, yielding continuity of the perturbation function $p(u)$, and in particular lower semicontinuity at $u=0$ (stability of degree~0) \cite{berge1997topological}.

Finally, growth condition together with stability of degree~0 are the necessary and sufficient conditions for strong duality of the augmented dual problem \ref{eq:dual_problem} \cite[Theorem 4]{RockaDual}.
\end{proof}

\subsection{Optimal Dual Solutions}
\label{sec:optimal_dual}

\begin{definition}[Second order stability]
The problem \ref{eq:constrained_risk} is stable of degree 2 if there is an open neighborhood $\mathcal{O}$ of the origin in $\mathbb{R}^m$, an element $\bar{\lambda} \in \mathbb{R}^m$ and a number $\bar{\alpha}>0$ such that
\begin{equation}
\label{eq:2stab}
    \begin{aligned}
        p(u) \geq p(0) - <\bar{\lambda},u> - \frac{\bar{\alpha}}{2}||u||^2 \quad ; \quad \forall u \in \mathcal{O}
    \end{aligned}
\end{equation}
with $p$ defined as \ref{eq:pert_func}.
\end{definition}

\begin{theorem}[Optimal Dual Solutions]
\label{thm:optimal_dual}
Suppose Assumptions~\ref{ass:l0_bounded}--\ref{ass:sosc} hold. 
Then there exist $(\bar{\lambda},\bar{\alpha})$ such that
\[
g(\bar{\lambda},\bar{\alpha}) = \max(D) = \inf(P).
\]

Moreover, any pair satisfies $g(\bar{\lambda},\bar{\alpha}) = \max(D)$ if and only if it satisfies \ref{eq:2stab}.
\end{theorem}

\begin{proof}
By Assumption~\ref{ass:l0_bounded}, the growth condition holds. Combined with Assumptions~\ref{ass:licq}--\ref{ass:sosc}, this implies that problem \ref{eq:constrained_risk} is stable of degree~2 by \cite[Theorem 6]{RockaDual}. Hence the hypotheses of \cite[Theorem 5]{RockaDual} apply, ensuring the existence of $(\bar{\lambda},\bar{\alpha})\in\mathcal{T}$ such that
$g(\bar{\lambda},\bar{\alpha})=\max(D)=\inf(P)$, as well as the stated equivalence with \ref{eq:2stab}.
\end{proof}

\section{PACC-Learnability}
\label{sec:gener}
\begin{theorem}[PAC-Constrained Dual Learning]
\label{thm:generalisation}
Let $(\hat{\lambda}^\star,\hat{\alpha}^\star)$ be a solution of the empirical augmented dual problem~\ref{eq:emp_dual_problem}. 
Under Assumptions~\ref{ass:l0_bounded}--\ref{ass:unif-conv}, the solution of~\ref{eq:emp_constrained_risk} satisfies 
$\hat{\theta}^\star \in \arg\min_{\theta \in \Theta} \hat{L}(\theta,\hat{\lambda}^\star,\hat{\alpha}^\star)$ and, with probability $1-(2m+1)\delta$,
\begin{align}
|P^\star - \ell_0(f_{\hat{\theta}^\star})| 
&\;\le\; (2+\Delta_{\lambda})\bar{\zeta} + \Delta_{\alpha}m\bar{\zeta}^2, 
\label{eq:generalisation-obj} \\
\ell_i(f_{\hat{\theta}^\star}) 
&\;\le\; \zeta_i(N_i,\delta), \quad \forall i=1,\dots,m,
\label{eq:generalisation-constr}
\end{align}
where $P^\star$ is the optimal value of~\ref{eq:constrained_risk}, 
$\bar{\zeta} = \max_i \zeta_i(N_i,\delta)$, 
$\Delta_\lambda = \max(\|\lambda^*\|_1,\|\hat{\lambda}^*\|_1)$, 
and $\Delta_\alpha = \max(\alpha^*,\hat{\alpha}^*)$.
\end{theorem}

\begin{proof}
As $(\hat{\lambda}^\star,\hat{\alpha}^\star)$ solves \ref{eq:emp_dual_problem}, \cite[Corollary 3.2]{RockaDual} implies that any solution of \ref{eq:emp_constrained_risk} satisfies $\hat{\theta}^\star \in \arg\min_{\theta \in \Theta} \hat{L}(\theta,\hat{\lambda}^\star,\hat{\alpha}^\star)$. It remains to establish near feasibility \ref{eq:generalisation-constr} and near optimality \ref{eq:generalisation-obj} for the statistical problem.

\textbf{Near Feasibility.} By Assumption~\ref{ass:unif-conv}, with probability $1-\delta$,
\begin{equation}
|\,\ell_i(f_{\theta}) - \hat{\ell}_i(f_{\theta})| \;\le\; \zeta_i(N,\delta)
\end{equation}
for all $\theta$. In particular,
\begin{equation}
\begin{aligned}
\label{eq:near-feas-i}
\,\ell_i(f_{\hat{\theta^*}})  &\;\le\;  \hat{\ell}_i(f_{\hat{\theta^*}}) + \zeta_i(N,\delta) \\
&\;\le\; \zeta_i(N_i,\delta) 
\end{aligned}
\end{equation}
where the last inequality uses that $\hat{\theta}^\star$ is a solution (i.e., feasible) for \ref{eq:emp_constrained_risk}. A union bound then gives that \ref{eq:near-feas-i} holds simultaneously for all $i$ with probability at least $1-m\delta$.

\textbf{Near Optimality.}
We prove the two-sided bounds
\begin{equation}
\begin{aligned}
\label{eq:bounds}
-(2+\|\hat\lambda^*\|_1)\bar{\zeta} - \hat\alpha^*m\bar{\zeta}^2 &\;\le\; P^\star - \ell_0(f_{\hat{\theta}^\star})  \\
 P^\star - \ell_0(f_{\hat{\theta}^\star}) &\;\le\; \|{\lambda}^*\|_1\bar{\zeta} + {\alpha}^*m\bar{\zeta}^2
\end{aligned}
\end{equation}
where $(\lambda^*,\alpha^*)$ solves \ref{eq:dual_problem}, $(\hat \lambda^*, \hat \alpha^*)$ solves \ref{eq:emp_dual_problem}, and $\bar{\zeta}=\max_i \zeta_i(N_i,\delta)$. Existence of these pairs follows from Assumptions \ref{ass:l0_bounded}–\ref{ass:sosc} and Theorem~\ref{thm:optimal_dual}. Set $u_{\zeta}=\mathbf{1}_m\bar{\zeta}$. By near feasibility, with probability at least $1-m\delta$,
\begin{equation}
\label{eq:24}
\,\ell_i(f_{\hat{\theta^*}})  \;\le\; \zeta_i(N_i,\delta) 
\end{equation}
for all $i=1,\dots,m$. Hence $\hat{\theta}^\star$ is feasible for \ref{eq:pert_func} under the relaxation $u=u_{\zeta}$. Therefore,
\begin{equation}
\begin{aligned}
\label{eq:upper-bound}
P^* = p(0) &\leq p(u_{\zeta}) + <\lambda^*,u_{\zeta}> + \alpha^* \|u_{\zeta}\|^2 \\
&\leq  \ell_0(f_{\hat{\theta}^\star}) + <\lambda^*,u_{\zeta}> + \alpha^* \|u_{\zeta}\|^2 \\
&\leq \ell_0(f_{\hat{\theta}^\star}) + \bar{\zeta}\|\lambda^*\|_1 + \alpha^* m\bar{\zeta}^2
\end{aligned}
\end{equation}
with probability at least $1-m\delta$. The first inequality is by Theorem~\ref{thm:optimal_dual}; the second uses suboptimality of the feasible point $f_{\hat{\theta}^\star}$ for the relaxed problem.

By the analogous argument to \ref{eq:near-feas-i}, with probability $1-m\delta$,
\begin{equation}
\begin{aligned}
\label{eq:near-feas-i-erm}
\,\hat \ell_i(f_{{\theta^*}})  &\;\le\;  {\ell}_i(f_{{\theta^*}}) + \zeta_i(N,\delta) \\
&\;\le\; \zeta_i(N_i,\delta) 
\end{aligned}
\end{equation}
for all $i=1,\dots,m$, where the second inequality uses feasibility of $\theta^*$ for \ref{eq:constrained_risk}. Thus $\theta^*$ is feasible for \ref{eq:emp_pert_func} with relaxation $u=u_\zeta$, and so
\begin{equation}
\begin{aligned}
\label{eq:26}
\hat{\ell_0}(f_{{\theta}^*}) = \hat p(0) &\leq \hat p(u_{\zeta}) + <\hat \lambda^*,u_{\zeta}> + \hat \alpha^* \|u_{\zeta}\|^2 \\
&\leq  \hat \ell_0(f_{{\theta}^\star}) + <\hat \lambda^*,u_{\zeta}> + \hat \alpha^* \|u_{\zeta}\|^2 \\
&\leq \hat \ell_0(f_{{\theta}^\star}) + \bar{\zeta}\|\hat \lambda^*\|_1 + \hat \alpha^* m\bar{\zeta}^2
\end{aligned}
\end{equation}
with probability at least $1-m\delta$. The first inequality is by Theorem~\ref{thm:optimal_dual}; the second uses suboptimality of the feasible point $f_{{\theta}^\star}$ for the relaxed problem.
Moreover, for both $\hat{\theta}^*$ and $\theta^*$ we have
\begin{equation}
\label{eq:27}
|\hat{\ell}_0(f_{{\theta}}) - \ell_0(f_{{\theta}})| \leq \zeta_0(N_0,\delta) \leq   \bar{\zeta}
\end{equation}
with probability $1-\delta$. Combining \ref{eq:26} and \ref{eq:27} and applying a union bound yields
\begin{equation}
\begin{aligned}
\label{eq:lower-bound}
    \ell_0(f_{\hat{\theta}^*}) &\leq \ell_0(f_{\theta^\star}) + \bar{\zeta}\|\hat \lambda^*\|_1 + \hat \alpha^* m\bar{\zeta}^2 + 2\bar \zeta \\
     &= P^* + (2+\|\hat \lambda^*\|_1)\bar{\zeta} + \hat \alpha^* m\bar{\zeta}^2
\end{aligned}
\end{equation}
with probability at least $1-(m+2)\delta$.

A final union bound shows that \ref{eq:upper-bound} and \ref{eq:lower-bound} hold simultaneously with probability $1-(2m+2)\delta$. Note that \ref{eq:upper-bound} already conditions on \ref{eq:near-feas-i}, so no additional union bound is required to couple near feasibility with near optimality. Taking the maximum absolute value in \ref{eq:bounds} gives \ref{eq:generalisation-obj}.
\end{proof}

\end{document}